\documentclass[sts
]{imsart}
\usepackage[utf8]{inputenc}
\makeatletter
\let\c@author\relax
\makeatother
\usepackage[style=alphabetic,maxbibnames=99,giveninits=true,url=false,isbn=false,doi=false,related=false]{biblatex}
\usepackage{amsthm,amsmath, amssymb, color}
\RequirePackage[colorlinks,citecolor=blue,urlcolor=blue]{hyperref}

\usepackage{graphicx, pstricks, pst-node, color}
\usepackage{psfrag,color}
\usepackage{dcolumn}
\usepackage{subfigure}
\usepackage{flafter, bm, dsfont}
\usepackage[section]{placeins}

\usepackage{url}
\usepackage{graphics}
\usepackage[normalem]{ulem}
\usepackage{overpic}

\newtheorem{thm}{Theorem}
\newtheorem*{thm*}{Theorem}
\newtheorem{rem}[thm]{Remark}
\newtheorem{defi}[thm]{Definition}

\renewcommand{\P}{\mathbb{P}}	
	
\newcommand{\esp}{\mathbb{E}}
\newcommand{\E}{\esp}	

\newcommand{\ZZ}{\mathbf{Z}}

\newcommand{\cL}{\mathcal{L}}
\newcommand{\R}{\mathbb R}

\newcommand{\X}{\mathcal{X}}

\newcommand{\inde}{{\perp\!\!\!\perp}}
\newcommand{\ra}{\rightarrow}

\newcommand{\DI}{\mathsf{DI}}
\newcommand{\PC}{\mathrm{PC}}
\newcommand{\ud}{\mathrm{d}}
\renewcommand{\phi}{\varphi}

\newcommand{\GG}{\mathcal{G}}
\newcommand{\FF}{\mathcal{F}}
\newcommand{\RR}{\mathcal{R}}
\newcommand{\EE}{\mathcal{E}}\newcommand{\cE}{\mathcal{E}}

\newcommand{\cP}{\mathsf{P}}

\renewcommand{\cE}{\EE}

\renewcommand{\abstractname}{Abstract}

\definecolor{pink}{cmyk}{0, 1, 0, 0}

\begin{document}
\begin{frontmatter}
\title{Projection to Fairness\\ in Statistical Learning}
\runtitle{Fairnness in Statistical Learning}

		\author{ 
			\fnms{Thibaut}
			\snm{Le Gouic}\thanksref{t1}\ead[label=tlg]{thibaut.le\_gouic@math.cnrs.fr}
			\fnms{Jean-Michel} \snm{Loubes}\thanksref{t3}\ead[label=jml]{loubes@math.univ-toulouse.fr}
			\and
			\fnms{Philippe} \snm{Rigollet}\thanksref{t2}\ead[label=rigollet]{rigollet@math.mit.edu}
		}
		
		\affiliation{MIT and University Toulouse 3}
 		\thankstext{t1}{Department of Mathematics, Massachusetts Institute of Technology and Aix Marseille Univ, CNRS, Centrale Marseille, I2M, Marseille, France}
		\thankstext{t3}{Department of Mathematics, LIDS, Statistics and Data Science Center, Massachusetts Institute of Technology.
		}
		\thankstext{t2}{Institut de Math\'ematiques de Toulouse, University Toulouse 3.}

		\address{{Thibaut Le Gouic}\\
			{Department of Mathematics} \\
			{Massachusetts Institute of Technology}\\
			{77 Massachusetts Avenue,}\\
			{Cambridge, MA 02139-4307, USA}\\
			\printead{tlg}
		}

				\address{{Jean-Michel Loubes}\\
			{Institut de Math\'ematiques de Toulouse} \\
			{University Toulouse 3}\\
			{Toulouse, France}\\
			\printead{jml}
		}

		\address{{Philippe Rigollet}\\
			{Department of Mathematics} \\
			{Massachusetts Institute of Technology}\\
			{77 Massachusetts Avenue,}\\
			{Cambridge, MA 02139-4307, USA}\\
			\printead{rigollet}
		}

\begin{abstract}
In the context of regression, we consider the fundamental question of making an estimator fair while preserving its prediction accuracy as much as possible. To that end, we define its projection to fairness as its closest fair estimator in a sense that reflects prediction accuracy.
Our methodology leverages tools from optimal transport to construct efficiently the projection to fairness of any given estimator as a simple post-processing step. Moreover, our approach precisely quantifies the cost of fairness, measured in terms of prediction accuracy.
    \end{abstract}
\end{frontmatter}


\section{Introduction}

Machine learning has become pervasive in most aspects of modern society, leading to breakneck advances in data-driven decision making. This rapid development also comes with pressing moral and legal issues, especially in sensitive areas such as justice and human resources that are tightly knit into the very fabric of modern societies. We refer to \cite{Dresseleaao5580,chouldechova2017fair, barocas2016big,kim2016data} and references therein for concrete examples of unfair treatment in automatic decision making. Unfairness in machine learning algorithms may come from many sources ranging from historically biased training data to decision rules that favor majority groups~\cite{friedler2019comparative,feldman2015certifying,besse2020survey}.
These examples raise the fundamental question of \emph{fairness} in machine learning: can we design efficient machine learning that provably guarantee a fair treatment for all? \\
Actually, fairness is defined with respect to a so-called {\it sensitive variable} such as gender or ethnic origin or more generally which refers to a characteristic  that \emph{should not} play a decisive role in the decision making process for legal, ethical or practical reasons. Many definitions of fairness have arisen in machine learning to quantify the relationship between the behavior of an algorithm and the sensitive variable. Hence  a \emph{fair} algorithm is designed to remove partially or totally the influence of $S$ in its outcome, which may alter its behaviour. Note that because of correlations, it is possible for an algorithm to be unfair with respect to a variable $S$ that is not even part of its input~\cite{besse2020survey,reuters}.

The widespread adoption of fair algorithms is primarily hindered by an inherent trade-off between fairness and accuracy. Indeed, to become fair, an algorithm must sacrifice prediction accuracy, a measure that has been the gold standard to compare various methods since the dawn of machine learning. As a result, it is not only important to develop mechanisms that achieve fairness but also to quantify and minimize the \emph{cost of fairness} in terms of prediction accuracy. 

In this work, we use tools from the theory of optimal transport to study fairness in the central statistical problem of  {regression} by developing a \emph{projection to fairness} and precisely quantifying the \emph{cost of fairness}. Recall that the goal of regression is to predict a random variable $Y \in \R$ given a pair $(X,S)$, where $S$ denotes the sensitive variables and $X$ denotes the remaining  non-sensitive variables. A \emph{regressor} is a function $(X,S)\mapsto g(X,S) \in \R$ and we measure its performance using the quadratic risk:
\[
\RR(g):=\esp\big[|Y-g(X,S)|^2\big].
\] 
Given a class $\GG$ of regressors, we denote by $\RR(\GG)$ the best possible risk over this class:
\begin{equation}
\label{eq:riskbayes}    
\RR(\GG):=\inf_{g\in\GG}\RR(g).
\end{equation}
When $\GG=\FF$ is the set of all measurable functions, it is well known that the infimum in~\eqref{eq:riskbayes} is achieved by the \emph{regression function} $\eta(x,s):=\esp[Y|X=x,S=s]$. In this context, it is also called the \emph{Bayes regressor}. To emphasize the special role of the variable $S$, we write $\eta_s(x)=\eta(x,s)$ in the rest of the paper. In this work, we are concerned with the \emph{excess-risk} $\EE(\GG)=\RR(\GG)-\RR(\FF)\ge 0$ of a class $\GG$ constrained to contain only \emph{fair regressors}. In other words, the excess risk $\EE(\GG)$ quantifies the \emph{cost of fairness} in regression.

\medskip

Many definitions of fairness have been considered in the literature, originally for classification problems. In the context of regression, the fairness of a candidate regressor $g$ is characterized by the statistical dependence between  $g(X,S)$ and $S$. This approach has led to methodological choices that aim at quantifying this notion of dependence or correlation in order to remove it as a way to obtain fair regressors~\cite{fitzsimons2019general, Agarwal2019FairRQ, chzhenPlug-in2020}.

In this work we focus on a mild generalization of {\it demographic parity} (DP) as a notion of fairness; see Definition~\ref{def:PC}. A regressor $g$ is said to be \emph{DP-fair} if $g(X,S)$ is independent of the variable $S$. As a result, the decision of a DP-fair regressor is not impacted by the sensitive variable $S$. When $Y \in \{0,1\}$ (binary classification) and $S \in \{0,1\}$,  DP boils down to the following condition:
\[
\P(g(X,S)=1|S=1)=\P(g(X,S)=1|S=0).
\]
This condition may be relaxed by controlling the ratio of the two terms above, known as the \emph{disparate impact} of $g$:
\begin{equation}
\label{eq:DI}
\mathsf{DI}(g):= \frac{\P(g(X,S)=1|S=1)}{\P(g(X,S)=1|S=0)}.
\end{equation}
We refer to \cite{feldman2015certifying} and references therein for a description of this measure which has become popular due to its interpretation in terms of legislation.  

\medskip

Once unfairness is measured, it needs to be corrected. Methods aiming at reducing unfairness may be divided into three main categories : pre-processing the observations to remove the influence of the sensitive variable, constraining the optimization problem~\eqref{eq:riskbayes} or post-processing the outcome of the algorithm; see \cite{Oneto2020,2015arXiv150705259B,pmlr-v81-menon18a} and references therein for a review of these methods. Many of the existing methods hinge around the following task: comparing the distributions of an object (observations, regressor or scores of a regressor) for the different values that the sensitive variable $S$ can take. In other words, this task consists in comparing   different conditional distributions given $S$. In particular, if two such conditional distributions are close with respect to   a well chosen distance, then the sensitive variable plays little role in that object.
Of course, the choice of the distance between said conditional distributions is key and Wasserstein distances from the theory of optimal transport have played a preponderant role in this context. In fact, many works at the interface between optimal transport and fairness have been conducted in this direction, including methodological and theoretical treatments \cite{feldman2015certifying,Hacker2017ACF,del2018obtaining,del2019central}. Particularly relevant to the present paper is the work of  \cite{jiang2019wasserstein} on fair classification, where  a bound on the classification risk which involves the 1-Wasserstein distance is established. In turn, this bound is minimized by a certain $1$-\emph{Wasserstein barycenter}.
This result enables to post-process the scores of a classifier to gain fairness.

\subsection*{Our contributions.} In this work, we completely characterize the cost of DP-fairness for regression.  More specifically, we establish matching upper and lower bounds on the excess risk of a DP-fair regressor.  Our upper bound is achieved via a $2$-Wasserstein barycenter that can be computed efficiently using a multimarginal formulation. Importantly, while our optimal upper bound is achieved by post-processing a specific optimal estimator, the Bayes regressor, our methodology can be used to post-process \emph{any} estimator. In fact, our lower bound also extends to more general classes of fair algorithms that are characterized by their conditional distribution given $S$ and that encompass DP-fair regressors as a canonical example; see Definition~\ref{def:PC}.

The rest of the paper is organized as follows.
We introduce notation and definitions of fairness we consider in Section~\ref{sec:fair}.
In Section~\ref{s:bound}, we establish a lower bound on the excess risk of DP fair regressors. 
Section~\ref{s:build} focuses on building an optimal DP-fair regressor.
Section~\ref{s:estimation} is devoted to the estimation of our optimal regressor.
Proofs are gathered in Appendix~\ref{s:append}.

\section{Fair regressors}
\label{sec:fair}

Fix a positive integer $k$ and let $(Y,X,S)\sim \P$ be a  random triple taking values in $\R \times \R\times [k]$, where $\X$ is an abstract topological space and $[k]:=\{1, \ldots, k\}$; write $\pi_s=\P(S=s),$ for $s \in [k]$. Finally, let $\GG$ be a class of regressors $g: \X\times [k] \to \R$.

To set a benchmark, consider first the case where $\GG=\FF$, the set of all measurable functions from $\X\times [k]$ to $\R$. In this case, the optimal risk (a.k.a. Bayesian risk), is defined as 
\[
\RR^\star:=\RR(\FF)=\min_{g\in\FF}\esp\big[|Y-g(X,S)|^2\big],
\]
and the minimum is achieved for the Bayes regressor $g(X,S)=\eta_S(X)=\E[Y|X,S]$. The \emph{excess-risk} of the class $\GG$ of regressors is then given by 
\begin{equation}
\label{eq:excessr}
\EE(\GG):=\RR(\GG) - \RR^\star.
\end{equation}
To model fairness, we consider the following class of regressors that encompasses several existing fair regressors (and classifiers). In the following definitions, we allow for \emph{randomized} regressors  $g:\X\times[k] \to \R$ equipped with an exogenous source of randomness that is independent of $(Y, X,S)$.

\begin{defi}[Fairness profile]
\label{def:FP}
The \emph{fairness profile} of a (possibly random) regressor $g:\X\times[k] \to \R$ is the vector $\cP(g):=\big(\cL(g(X,S)|S=s)\big)_{\, s \in [k]}$ of conditional distributions of $g(X,S)$ given $S=s$ for all values of $S$. 
The map $g \mapsto \cP(g)$ is called the \emph{profiling} map.
\end{defi}

We focus on classes of regressors that are complete under profiling. 
\begin{defi}[PC class]
\label{def:PC}
A class $\GG\subset \FF$ of (possibly random) regressors $g:\X\times[k] \to \R$ is \emph{profile complete} for the model $(X,Y,S)$ (abbreviated $\PC(X,Y,S)$ or simply $\PC$) if any measurable function $h \in \FF$  that shares the same fairness profile with some $g \in \GG$ is also itself in $\GG$. 
\end{defi}


\noindent PC classes encompass the following notions that are used in the fairness literature.
\begin{itemize}
\item {\bf Demographic parity}.
The set $\GG_\inde$ of DP-fair regressors $g$ such that $g(X,S)$ is independent of $S$ is a canonical example of a PC class. Indeed,   $g(X,S)$ is independent of $S$ if and only if the conditional distribution of $g(X,S)$ given $S$ is equal to the distribution of $g(X,S)$.
In particular, the profile $\cP(g)$ of $g \in \GG_\inde$ is a constant vector $(\cL(g(X,S)))_{s\in[k]}$ and if $h$ has profile $\cP(g)$, then $h(X,S)$ is independent of $S$ so that $h \in \GG_\inde$.

\item {\bf Bounded conditional variance}. Demographic parity requires that the predictions of the algorithm are the same for each of the $k$ values of $S$. In particular, $\esp[g(X,S)|S]=\esp[g(X,S)]$ almost surely in $S$ so that the conditional distribution of $g(X,S)$ given $S$ has zero variance. This condition may be weakened by instead requiring that  $\esp[g(X,S)|S]$ has small variance, which warrants a similar behavior for all values of $S$; see \cite{2019arXiv190108665W} where variability of the conditional loss of the regressor is used as a measure of fairness.
To that end, fix $\alpha>0$ and let $\GG_\alpha$ denote the class of functions $g$ such that ${\rm Var}(\esp[g(X,S)|S]) \leq \alpha$. This class is a clearly PC class since it is only characterized by the fairness profile of $g$.

\item {\bf Fixed Disparate impact}.
If $Y \in \{0,1\}$ (binary classification), the set $\GG_\alpha$ of classifiers $g$ with disparate impact bounded by $\alpha >0$, i.e. $\DI(g)\le\alpha$, is also a PC class.
To see this, note that the condition $\DI(g)=\alpha$  depends only on the fairness profile of $g$.
\end{itemize}


\section{A lower bound on the cost of fairness}
 \label{s:bound}
 
 Our results use the notion of optimal transport between probabilities. We briefly recall here the notion of Wasserstein distance and refer the reader to the manuscripts~\cite{villani_topics_2003,villani_optimal_2008, santambrogio_optimal_2015} for a more comprehensive treatment.

Let $P$ and $Q$ be two probability measures on $\R$ and let $\Pi(P,Q)$ denote the set of joint probability measures on $ \mathbb{R} \times \mathbb{R}$ with marginals given by $P$ and $Q$. The squared Wasserstein distance between
$P$ and $Q$ is defined as
\[
W_2^2(P,Q):=\min_{\pi \in\Pi(P,Q)} \int |x-y|^2 d\pi(x,y)=\int_0^1|F^{-1}(t)-G^{-1}(t)|^2\ud t\,,
\]
where $F^{-1}$ and $G^{-1}$ denote the quantile function (a.k.a inverse CDF) of $P$ and $Q$ respectively.

Recall that when $\GG$ is the set $\FF$ of all measurable functions, the optimal regressor is the Bayes regressor $ \eta_S(X)=\esp(Y|X,S)$. Let $\mu_s$ denote the conditional distribution of $\eta_S(X)$ given $S=s$.
Fairness conditions tend to impose that the distribution of $g(X,S)$ conditionally to $S=s$ are close for different values of $s$. In other words, $g$ is fair if its fairness profile exhibits low variability across its entries. Hence it is natural that the variability of the random probability measure $\mu_S$ plays en essential role to quantify the cost of fairness.

We make this intuition precise in the following theorem. It also highlights the central role of the $2$-Wasserstein distance in quantifying the randomness of $\eta_S$ as a random probability measure.
\begin{thm} \label{th:bound}
Fix a class $\GG$ of regressors and for any $g$ in $\GG$, let $\nu_S(g)$ denote the conditional distribution of $g(X,S)$ given $S$.   Then, the excess-risk of $\GG$ is lower bounded by
\begin{equation} \label{eq:bound}
\EE(\GG) \geq \inf_{g \in \GG}\sum_{s=1}^{k}\pi_s W_2^2(\mu_s,\nu_s(g)).
\end{equation}
Moreover, if $\GG$ is $\PC$ and $\mu_S$ has density w.r.t. the Lebesgue measure almost surely in $S$, then \eqref{eq:bound} becomes an equality
\begin{equation} \label{eq:equal}
\EE(\GG) = \inf_{g \in \GG}\sum_{s=1}^{k}\pi_s W_2^2(\mu_s,\nu_s(g)).
\end{equation}
\end{thm}

Theorem~\ref{th:bound} precisely quantifies the cost of fairness for PC classes using the squared Wasserstein distance. More specifically, the cost of fairness is precisely given by the minimum over $g \in \GG$ of the variance of a random entry of the fairness profile of $g$. 
This observation has significant practical consequences: fair regressors can  be found by minimizing \eqref{eq:equal}.

In the next section we specialize this bound to the case of DP fairness where the right-hand side of~\eqref{eq:equal} translates into a Wasserstein barycenter.


This remark enables to build a new fair classification for regression type problems in Section~\ref{s:build}.

\begin{rem}
For the quadratic regression (i.e. quadratic loss for the risk), the excess risk and the average squared distance between a fair regressor $g(X,S)$ and the Bayes regressor $\eta_S(X)$ coincide (see equation \eqref{eq:risk-est}).
This is the only part where we use the particular choice of quadratic regression.
If we were interested in the quantity $\esp[d(g(X,S),\eta_S(X))^2]$ to quantify the performance of a fair regressor, for some metric $d$, then similar results hold for any kind of regression.
\end{rem}
\begin{rem}
Note that Theorem~\ref{th:bound} provides also a bound for the excess risk in the classification case. Let $Y \in \{0,1\}$ and  $\mathcal{G}$ be a subset of functions with also binary values $\{0,1\}$.
Note that
\[
\esp \big[|Y-g(X,S)|^2\big] = \P(Y \neq g(X,S))
\]
is the classification risk, while $\eta_S(X)=\esp[Y|S,X]=\P(Y=1|X,S)$.
Hence \eqref{eq:equal} provides a control over 
\[
\P(Y \neq g(X,S))- \esp \big[|Y-\eta_S(X)|^2 \big]
\]
which differs from the excess risk

\[
\cE(\{g\})=\P(Y \neq g(X,S))- \inf_g \P(Y \neq g(X,S))
\]
since $\eta_S(X) \notin \{0,1\}$. \\
\indent Yet this bound can still be used when trying to understand the prediction of scores used in classification before a threshold is applied.
\end{rem}

\section{Projection to DP-fairness} \label{s:build}
In this section, we specialize the results of the previous section to the case of DP-fairness. In this case, we show that a regressor $g$ may be \emph{projected to fairness} using Wasserstein barycenters.

\subsection{Achieving fairness using Wasserstein barycenters}
Recall that  $\GG_\inde$ denotes the class of DP-fair regressors $g$ such that $g(X,S)$ is independent of $S$. This implies that $s\mapsto \nu_s(g)$ is constant equal to $\nu(g)$, or in other words that the fairness profile of of $g$ is the constant vector $\cP(g)=(\nu(g), \ldots, \nu(g))$ for all $g \in \GG_\inde$. Moreover, since $\GG_\inde$ is PC, Theorem~\ref{th:bound} yields
\begin{equation}\label{eq:optisbary}
\EE(\GG_\inde) = \inf_{g \in \GG_\inde} \sum_{s=1}^k \pi_s W_2^2(\mu_s,\nu(g)),
\end{equation}

Note that, as $g$ ranges through $\GG_\inde$, the measure $\nu(g)$, which is the distribution of $g(X,S)$ ranges through all possible measures over $\R$ \emph{if we allow for randomized regressors $g$}. As a result, we have
\begin{equation}
     \label{eq:levrai}
      \inf_{g \in \GG_\inde} \sum_{s=1}^k \pi_s W_2^2(\mu_s,\nu(g))= \inf_{\nu} \sum_{s=1}^k \pi_s W_2^2(\mu_s,\nu),
\end{equation}
where the infimum is taken over all probability measures on $\R$.

We recognize the problem of finding the \emph{Wasserstein barycenter} of the distribution $P_{\mu_S}=\sum_{s=1}^k \pi_s \delta_{\mu_s}$ generating the mixture of conditional distributions $\mu_s$ with weights $\pi_s$.
When it exists, a minimizer of \eqref{eq:levrai} is called a \emph{barycenter} of the empirical distribution $P_{\mu_S}$ and  denoted~$\bar{\nu}$.

Wasserstein barycenters were first introduced in \cite{agueh_barycenters_2011} and they have since generated sustained research activity due to their wide applicability. Several algorithms to compute Wasserstein barycenters have been developed \cite{cuturiSinkhornDistancesLightspeed2013,cuturiFastComputationWasserstein2014,solomonConvolutionalWassersteinDistances2015} and statistical guarantees of its solutions have been studied in \cite{ahidar-coutrixConvergenceRatesEmpirical2019,chewiGradientDescentAlgorithms2020a,legouicFastConvergenceEmpirical2019,boissard2015distribution,le_gouic_existence_2017,ALVAREZESTEBAN2016744,kroshninStatisticalInferenceBuresWasserstein2019}. 

As pointed out in  \cite[Sec. 4]{agueh_barycenters_2011},  the barycenter problem is equivalent to the following multi-marginal problem.
Let $\Gamma(\mu_1,\dots,\mu_k)$ denote the set of joint probability measures on $\R^k$ with marginals given by $\mu_1,\dots,\mu_k$. The multi-marginal problem consists in the following minimization problem 
\begin{equation}\label{eq:multimarginal}
\min_{\gamma\in\Gamma(\mu_1,\dots,\mu_k)}\left\{ \sum_{s=1}^k\pi_s\int |z_s-b(\mathbf{y})|^2\ud\gamma(z_1,\dots,z_k)\right\} ,
\end{equation}
where $\mathbf{z} = (z_1,\dots,z_k) \in \R^k$  and $b: \R^k\to\R$ is the barycenter map defined by
\[
b(\mathbf{z}):=\sum_{s=1}^k\pi_s z_{s}.
\]

It follows from \cite[Prop. 4.2]{agueh_barycenters_2011} that a solution  $\gamma_\star$ of \eqref{eq:multimarginal} always exist.  In fact $\gamma_\star$ 
of the random vector $(F_1^{-1}(U),\cdots,F_k^{-1}(U))$ where $U \sim \mathsf{Unif}([0,1])$ and, for each $s \in [k]$,  $F_s^{-1}$ denotes the quantile fuction of $\mu_s$.

Moreover, the barycenter $\bar \nu$ of $P_{\mu_S}$ is the pushforward measure defined by $b_\#\gamma_\star := \gamma_\star \circ  b^{-1}$. Hence the following holds  
\begin{equation}
\inf_\nu  \sum_{s=1}^k\pi_s W_2^2(\nu,\mu_s)=\sum_{s=1}^k \pi_s\int|y_s-b(\mathbf{z})|^2\ud\gamma_{\star}(z_1,\dots,z_k).\label{eq:mingamma}
\end{equation}

Thus, if $\ZZ=(Z_1,\dots,Z_k) \sim \gamma_\star$, then  $b(\ZZ) \sim \bar \nu$, which is the distribution of the optimal fair regressor as shown in \eqref{eq:optisbary} according to Theorem~\ref{th:bound}. This multimarginal representation allows us  to exhibit an optimal DP-fair regressor.

\subsection{Projection to fairness of the Bayes regressor}

Assume that for some $s\in[k]$, the distribution $\mu_s$ admits a density w.r.t. Lebesgue measure on $\mathbb{R}^d$.
In this case, Theorem 4.1 in \cite{agueh_barycenters_2011} ensures that there exist measurable maps $T_s:\R\ra\R$ and $T^t:\R\ra\R$ which are optimal transport maps, pushing $\mu_s$ towards $\bar{\nu}$ and $\bar{\nu}$ towards $\mu_t$ respectively. 

Hence  if a random variable $Z_s$ has distribution $\mu_s$, then  $T_s(Z_s)$ has distribution $\bar{\nu}$ and for $s\in[k]$ such that $\mu_s$ has density w.r.t. the Lebesgue measure,
\[
\gamma_\star=(T^1\circ T_s,\dots,T^k\circ T_s)_\#\mu_s.
\]
Recall that $\eta_s(X) \sim \mu_s$ and assume hereafter that $\mu_s$ admits a density with respect to the Lebesgue measure for all $s \in [k]$. Since $b_\#\gamma_\star=\bar{\nu}$,   setting $T_s^t:=T^t \circ T_s$, we have that, for any $s\in[k]$ such that $\mu_s$ has density w.r.t. Lebesgue measure, 
\[
b\big(T_s^1( \eta_s(X)), \ldots,T_s^k( \eta_s(X)) \big)=\sum_{t=1}^k \pi_t T_s^t( \eta_s(X))\sim\bar{\nu}.
\]
In fact, since for any $s \in [k]$, it holds $F_s(\eta_s(X))=U \sim \mathsf{Unif}([0,1])$ and   $(F_1^{-1}(U),\cdots,F_k^{-1}(U)) \sim \gamma_\star$, we may also write
\[
\sum_{t=1}^k \pi_t F_t^{-1}\circ F_s( \eta_s(X))\sim\bar{\nu}
\]

The above construction yields a projection to fairness of the Bayes regressor.

of a random variable that has distribution $\bar \nu$ readily yields the following optimal DP-fair regressors.
\begin{thm}
\label{th:optpred}
For all $s \in [k]$, assume that the distribution $\mu_s$ of $\eta_s(X)$ admits a density w.r.t. the Lebesgue measure on $\R$ for all $s \in [k]$ and denote by $F_s$ its CDF. Then the regressor $g_\star$ defined by
\[
g_\star(x,s):=\sum_{t=1}^k\pi_t T_s^t(\eta_s(x))=\sum_{t=1}^k \pi_t F_t^{-1}\circ F_s( \eta_s(x))\sim\bar{\nu}\,, \qquad x \in \X, \ s \in [k]
\]
is optimal among DP-fair estimators: $\RR(g_\star) = \RR(\GG)$. Moreover, the cost of DP-fairness is given by
$$
\EE(\GG_\inde)=  \sum_{s=1}^k\pi_s W_2^2(\nu(g_\star),\mu_s)
$$
where $g_\star(X,S)\sim \nu(g_\star)$ is independent of $S$.
\end{thm}

This theorem provides an insight on how to build a fair procedure which preserves as much as possible of the prediction accuracy of the Bayes regressor. 
Recall that $X$ represents the characteristics of the individuals that are used to predict $Y$ and $S \in[k]$ represents the community or category to which they belong.
The conditional expectation $\eta_S(X)$ is the optimal regressor but the variable $Y$ is unfairly predicted by $\eta_S(X)$ in the sense that the distribution $\eta_S(X)$ depends on the sensitive variable $S$ --- thus violating the property of demographic parity. 
Then the solution of the multi-marginal problem $\gamma_{\star}$ gives a coupling between the random variables $\eta_1(X), \ldots, \eta_k(X)$. Since we assume that $\mu_s$ has a density for every $s \in [k]$, this coupling can be understood using transport maps as follows. 

For each individual $x \in \X$ in the category $S=s$ the Bayes regressor predicts $\eta_s(x)$. This prediction is then mapped to a prediction $\tilde \eta_t(x)$ for every other category $t\neq s$ using $\tilde \eta_t(x)=T^t_s(\eta_s(x))$. If the Bayes regressor was close to DP-fair then this prediction would not depend on $t$ and hence, we would have $\tilde \eta_t(x)\approx \eta_s(x)$ for all $t \neq s$. In general, the Bayes regressor may not be fair. Instead, the optimal DP-fair regressor $g_\star$ replaces the prediction $\eta_s(x)$ with the weighted average of all predictions $\tilde \eta_t(x), t \in [k]$ that are paired up with $\eta_s(x)$ using the multimarginal coupling.


Note that akin to the Bayes regressor, the DP-fair optimal regressor is the theoretical fair regressor since it depends on the unknown quantities $\eta_S$ and $\pi_S$. In practice, these have to be estimated.
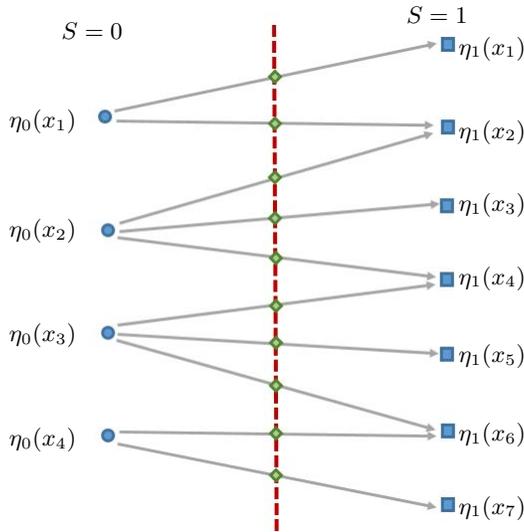
\begin{figure}[htbp]
		\centering
		\begin{overpic}[scale=.5,unit=1mm]{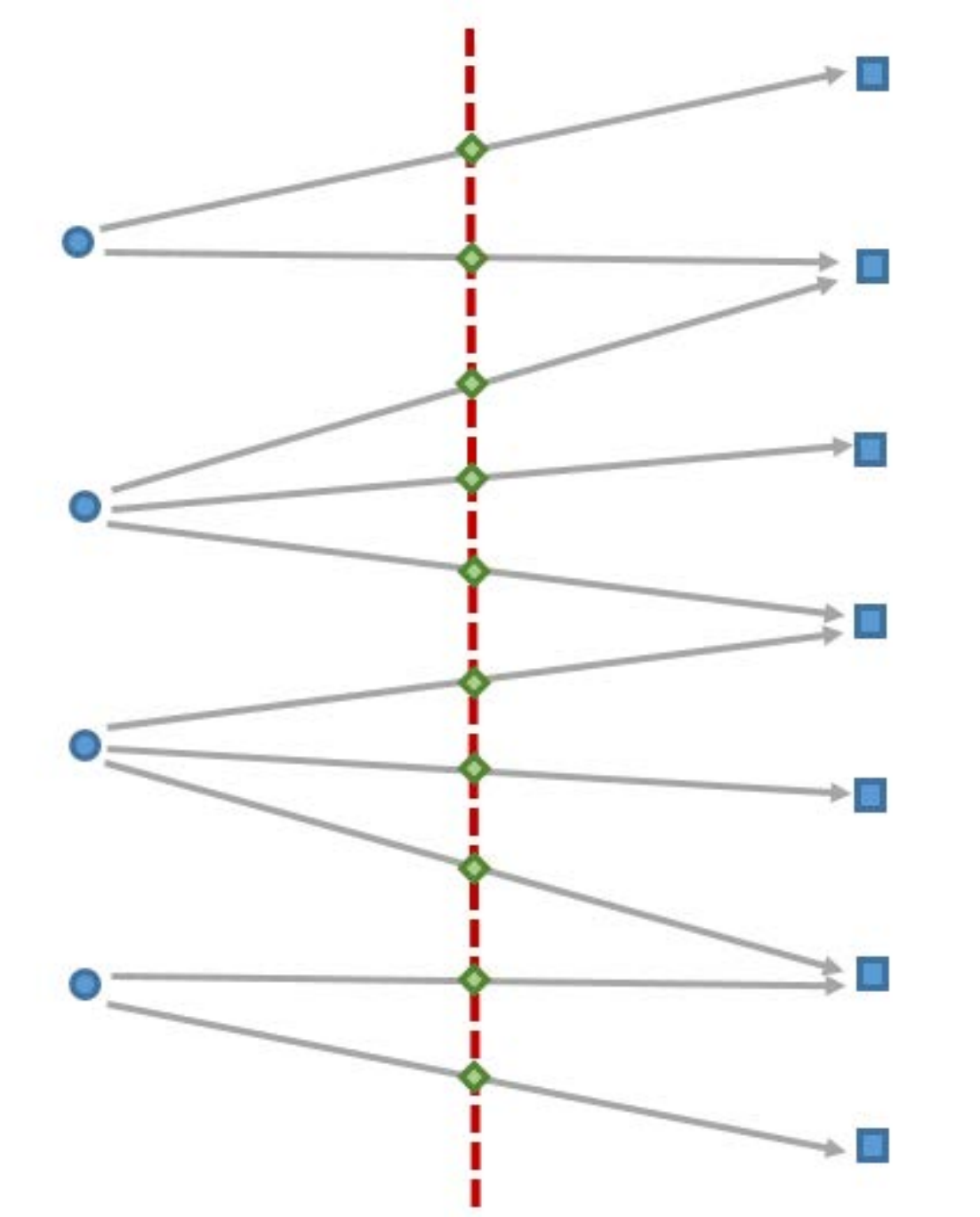}
			\put (-2,95){$S=0$}
			\put (64,98){$S=1$}
			\put (-12,78){$\eta_0(x_{1})$}
			\put (-12,57){$\eta_0(x_{2})$}
			\put (-12,37){$\eta_0(x_{3})$}
			\put (-12,17){$\eta_0(x_{4})$}
			\put (74,92){$\eta_1(x_{1})$}
			\put (74,76){$\eta_1(x_{2})$}
			\put (74,62){$\eta_1(x_{3})$}
			\put (74,48){$\eta_1(x_{4})$}
			\put (74,33){$\eta_1(x_{5})$}
			\put (74,18){$\eta_1(x_{6})$}
			\put (74,4){$\eta_1(x_{7})$}
		\end{overpic}
		\caption{Example of construction of a fair regressor as in Remark \ref{rem:noisygstar}. For $S=0$ and $i=1,\dots,4$, the values $\eta_0(x_i)$ of the unfair regressor represented by blue dots on the left are matched with values $\eta_1(x_j)$ of the unfair regressor for $S=1$. For each $i=1,\dots,4$, the optimal fair regressor for $X=x_i$ and $S=0$ is drawn among the green dots on the arrows emanating from $\eta_0(x_i)$.}
		\end{figure}

Shortly after the first version of this work was posted,~\cite{chzhenFairRegressionWasserstein2020} independently achieved similar results. Their result focuses on demographic parity for which they also use the multimarginal formulation of Wasserstein barycenters in order to make the Bayes regressor fair. 

\begin{rem}\label{rem:noisygstar}
A similar result holds when the distributions $\mu_s$ are not all absolutely continuous w.r.t. the Lebesgue measure.
In the case, the optimal regressor $g_\star$ can be randomized as follows.
Let $\ZZ=(Z_1, \ldots, Z_k) \sim \gamma_\star$ and for each $s \in [k]$, denote  by $\gamma^{(s)}_\star(Z_s)$ the conditional distribution of $\ZZ$ given $Z_s$.
Then, for $S=s$, the optimal DP-fair regressor $g_\star$ is given by  $g_\star(X,S)=b(\ZZ^S)$, where $\ZZ^s\sim \gamma^{(s)}_\star(\eta_s(X))$.
\end{rem}

\section{Estimation of optimal fair regressor}\label{s:estimation}

In practice, the exact distribution of $(X,Y,S)$ that is required to compute the fair regressor $g_\star$ is unknown.
Building on the previous section, we provide some guidelines  to construct an estimator of $g_\star$ using simple plug-in rules.

\subsection{General estimator of optimal fair regressor}
Theorem~\ref{th:optpred} provides a way to construct an optimal fair regressor $g_\star$ for the regression case.
Yet, it depends on unknown quantities, namely, the probability mass function $s \mapsto \pi_s$ of $S$, the transport maps $T_s$ and $T^t$ and $\eta_s(X)$ for each $s \in[k]$. 

In a statistical setup, rather than having access to these quantities, we observe independent copies $(Y_1,X_1,S_1), \ldots, (Y_n, X_n, S_n)$ of $(Y,X,S)$. We also assume a that we are in a  transductive learning setup were we have access to a collection of additional unlabeled features $(X_{n+1},S_{n+1}),\ldots, (X_{n+m}, S_{n+m})$ on which we aim to make predictions. 

For each $s \in [k]$, let $N_s=\#\{i\in\{1,\dots,n+m\}|S_i=s\}$ denote the number of observations in class $S$.

We propose the following estimator of $g_\star$. \vskip .1in
\noindent {\bf Algorithm to construct a fair regressor}
\begin{enumerate}
\item Estimate $x \mapsto \eta_s(x)$ for each $s\in[k]$ by an estimator of the conditional expectation  $\hat{\eta}_s(x)$ using  observations $(Y_i,X_i,S_i), i=1, \ldots, n$ . This can be achieved using parametric or non-paramteric regression techniques.
\item Approximate the distribution of each $\hat\eta_s(X)$ for $s=1,\dots,k$ by the empirical measure using the whole dataset $(X_i,S_i)_{i=1,\dots,n+m}$:
\[
\hat\mu_s:=\frac{1}{N_s} \sum_{i=1}^{N_s} \delta_{\hat{\eta}_s(X_i)}\mathbf{1}_{\{S_i=s\}}.
\]
\item Set 
\[
\hat \pi_s=\frac{N_s}{n+m} \qquad \text{and} \qquad \hat b:=(y_1,\dots,y_k)\mapsto \sum_{s=1}^k\hat \pi_s y_s
\]
and solve the multimarginal problem \eqref{eq:multimarginal} for distributions $\hat{\mu}_s$, which can be written as 
\begin{equation}\label{eq:multimarginal2}
\min_{\gamma\in\Gamma(\hat{\mu}_1,\dots,\hat{\mu}_k)}\left\{ \sum_{s=1}^k\hat\pi_s\int |y_s-\hat{b}(\mathbf{y})|^2\ud\gamma(y_1,\dots,y_k)\right\} ,
\end{equation}

Denote by $\hat\gamma$ its solution.
\item As in Remark \ref{rem:noisygstar}, denote by $\hat\gamma^{(s)}$ the conditional distribution of $\hat \gamma$ given its $s$-th coordinate.
For each $s\in[k], i \in [n+m]$ such that $S_i=s$, draw $\hat  \ZZ^s \sim \hat\gamma^{(s)}(\hat\eta_{s}(X_i))$ and set
\[
\hat g(X_i,S_i):=\hat b(\hat \ZZ^s).
\]
\end{enumerate}

This procedure provides an approximation of the fair regressor $g_\star$.
Note that $(x,s)\mapsto \hat g(x,s)$ is well defined only when $(x,s)$ lies in the predetermined set $(X_i,S_i)_{1\le i \le n+m}$ for which we want a fair regressor. Hence the whole procedure consists in post-processing the output of the algorithm for the different values of $S$ and aggregate them. \vskip .1in

In our estimation framework, the fairness profile of $\hat \eta$ is unknown as it depends on the unknown distribution of $(X,S)$. However, it can be estimated by \emph{empirical fairness profile} given by
\[
\widehat{\cP}_{n,m}(\hat \eta):=\left(\frac{1}{N_s}\sum_{i=1}^{n+m}\mathbf{1}_{\{S_i=s\}}\delta_{\hat \eta_{S_i}(X_i)}\right)_{s\in [k]}.
\]
By construction, the distribution of the empirically fair regressor $\hat g$ is the Wasserstein barycenter of the distributions in the fairness profile of $\hat\eta$, weighted by the empirical occurrences $\hat\pi_s$.
Therefore, if the empirical fairness profile of $\hat\eta$ converges to the fairness profile of $\eta$, using consistency of the barycenter (see \cite[Theorem 2]{le_gouic_existence_2017}), the distribution of $\hat g$ converges to the distribution of $g_\star$.
The following theorem describes such a situation.
\begin{thm}
\label{th:stability}
Suppose that $\hat\eta_s$ is a bounded $L_2$-consistent estimator of $\eta_s$ for each $s\in[k]$ and that $\hat\eta_s$ are uniformly Lipschitz for each $s\in[k]$.
Assume also that the distribution $\nu(g_\star)$ of $g_\star(X,S)$ is unique.
Then, as $n\to \infty$, the empirical distribution of $g(X_i,S_i)$ converges to $\nu(g_\star)$ in $W_2$, a.s..
\end{thm}

This approximation requires estimators of the Bayes predictor that can be found for instance in \cite{gyorfiDistributionfreeTheoryNonparametric2006, Tsy09}. 
There exist several algorithms to compute a multimarginal solution $\hat\gamma$ of \eqref{eq:multimarginal2} (see \cite{cuturiSinkhornDistancesLightspeed2013,cuturiFastComputationWasserstein2014,solomonConvolutionalWassersteinDistances2015, guminov2019accelerated,pmlr-v97-kroshnin19a,chewiGradientDescentAlgorithms2020a,lin2020fixedsupport,Dvi20,altschulerHighprecisionWassersteinBarycenters2020}). \vskip .1in

\begin{rem}
Recall from  Theorem~\ref{th:optpred}, that the optimal regressor that respects demographic parity is given by
\[
g_\star(x,s)=\sum_{t=1}^k\pi_t F_{t}^{-1}\circ F_{s}(\eta_s(x)).
\]
where $F_s$ denotes the the CDF of $\eta_s(X)$ and $F_S^{-1}$ denotes its quantile function. There is a vast literature that provides estimators $\widehat{F_s}$ and $\widehat{F_s^{-1}}$ for these functions. Using naive plug-in estimators yields the following estimator of $g_\star$
\[
\hat{g}(X,S) =\sum_{t=1}^k \hat{\pi}_j \widehat{F_s^{-1}}\circ \widehat{F_s}(\hat\eta_s(x)).
\]
\end{rem}

Hence, we have achieved the  construction of  a \emph{fair regressor} by considering the empirical barycenter of the approximated empirical distributions of the Bayes regressor.
So the procedure is thus a post-processing method that combines the contribution of  matching individuals with different sensitive attributes. \vskip .1in

Note that~\cite{chzhenFairRegressionWasserstein2020} also achieve finite sample results for similar plug-in estimators.

\section{Conclusion}
We have computed the exact loss of efficiency --- i.e the cost for fairness --- in the quadratic regression framework.
Our result shows that post-processing through optimal transport provides the optimal way to achieve demographic parity in this model.
Yet, some questions of interest remain.
As Theorem \ref{th:bound} strongly suggests, post-processing through optimal transport should also give an optimal regressor for fairness notions that can be expressed via PC classes.
What is the right corresponding multi-marginal problem in these cases of fairness, and how to solve them?
On a more technical note, Theorem \ref{th:stability} only expresses consistency of our estimator of the optimal predictor. What are the rates of convergence for this estimator?
      
\appendix  
\section{Proofs} \label{s:append}
\begin{proof}[Proof of Theorem~\ref{th:bound}]
Denote $\eta_S(X):=\esp [Y|(X,S)]$ and recall from Pythagoras' theorem that 
\begin{equation}
\label{eq:risk-est}
\esp|Y-\eta_S(X)|^2 + \esp|g(X,S)-\eta_S(X)|^2=\esp|Y-g(X,S)|^2.
\end{equation}
Therefore,
\[
\inf_{g\in\GG} \esp |Y-g(X,S)|^2    -\esp|Y-\eta_S(X)|^2=\inf_{g\in\GG}\esp \big[\esp(|g(X,S)-\eta_S(X)|^2|S)\big]
\]
For almost every value $s$ of $S$, the conditional distribution of $(g(X,S),\eta_S(X))$ given $S$ is a coupling between $\mu_s$ and $\nu_s$, hence by definition of the Wasserstein distance
\[
\esp[|g(X,S)-\eta_S(X)|^2|S] \geq \esp W_2^2(\nu_S(g),\mu_S).
\]
Integrating with respect to $S$ proves \eqref{eq:bound}.

If $\mu_s$ has density w.r.t.\! the Lebesgue measure, then there exists an optimal transport map $T_s:\R\ra\R$ such that  $T_s(\eta_s(X)) \sim \nu_s(g)$ and
\[
W_2^2(\nu_s(g),\mu_s)=\esp |T_s(\eta_s(X))-\eta_s(X)|^2.
\]
Define now $h:=(x,s)\mapsto T_s(\eta_s(x))$ and observe  $h(X,S)$ the same distribution $\nu_S(g)$ as $g(X,S)$. Hence, since $\GG$ is PC, it is also the case that $h \in \GG$. In particular it yields
$$
 \esp W_2^2(\nu_S(g),\mu_S) =  \esp |h(X,S)-\eta_s(X)|^2 \ge \inf _{g \in \GG}  \esp |g(X,S)-\eta_s(X)|^2\,.
$$
Thus \eqref{eq:bound} is an equality.

The case of noisy PC is handled similarly.
\end{proof}

\begin{proof}[Proof of Theorem~\ref{th:stability}]
Using consistency of the Wasserstein barycenter \cite[Theorem 2]{le_gouic_existence_2017}, we just need to prove that $\hat{P}_{\hat{\mu}_S}\to P_{\mu_S}$ in Wasserstein distance.
This is a consequence of $\hat\pi_s\to\pi_s$ --- which holds due to the law of large numbers, and $\hat\mu_s\to\mu_s$ in Wasserstein distance.
Therefore, it just remains to prove that, as $n+m\to\infty$,
\[
W_2^2(\hat\mu_s,\mu_s)\to0,\quad \forall s\in[k].
\]
Denote by $T^{n,m}:(x,s)\mapsto (T^{n,m}_x(x,s),T^{n,m}_s(x,s))\in\R\times\R$ the optimal transport from the distribution $\P_{(X,S)}$ of $(X,S)$ to the empirical measure 
\[
\P_{(X,S)}^{n+m}:=\frac{1}{n+m}\sum_{i=1}^{n+m}\delta_{(X_i,S_i)}.
\]
Denote by $L$ an upper bound of the Lipschitz constants of $\hat \eta_s$.
We can compute
\begin{align*}
W_2^2(\hat\mu_s,\mu_s)&\le \esp |\eta_S(X)-\hat\eta_{T_s(X,S)}(T_x(X,S))|^2\\
&\le  2\esp |\eta_S(X)-\hat\eta_{S}(X)|^2 +  2\esp \mathbf{1}_{T_s(X,S)=S}|\hat\eta_{S}(X)-\hat\eta_{S}(T_x(X,S))|^2\\
&\quad\, +2\esp\mathbf{1}_{T_s(X,S)\ne S}|\hat\eta_{S}(X)-\hat\eta_{T_s(X,S)}(T_x(X,S))|^2\\
&\le  2\esp |\eta_S(X)-\hat\eta_{S}(X)|^2+ 2L^2\esp |X-T_x(X,S)|^2 + 8M\P(T_s(X,S)\ne S),
\end{align*}
where $M$ is the bound of $\hat \eta_s$.
Now, $\esp |\eta_S(X)-\hat\eta_{S}(X)|^2\to0$ as $\hat\eta_s$ is $L_2$-consistent.
Then, 
\[
\esp |X-T_x(X,S)|^2=W_2^2(\P_{(X,S)},\P_{(X,S)}^{n+m})\to 0
\]
since the empirical measure is consistent in $W_2$ (this is a consequence of Varadarajan's Theorem on compact spaces, see for instance \cite{boissardMeanSpeedConvergence2014,weedSharpAsymptoticFinitesample2019} for rates of convergence).
Finally, since $|s-s'|\ge1$ for $s\ne s'\in[k]$, then, \[ \P(T_s(X,S)\ne S)\le  W_2^2(\P_{(X,S)},\P_{(X,S)}^{n+m})\to 0 \] a.s., which concludes the proof.
\end{proof}

\textbf{Acknowledgments}. \\
We thank Sinho Chewi for comments on an earlier version of this manuscript.
Thibaut Le Gouic was supported by ONR grant N00014-17-1-2147 and NSF IIS-1838071.
Jean-Michel Loubes thanks the AI interdisciplinary institute ANITI, grant agreement ANR-19-PI3A-0004 under the French investing for the future PIA3 program.
Philippe Rigollet was supported by NSF awards IIS-1838071, DMS-1712596, DMS-TRIPODS-1740751, and ONR grant N00014-17- 1-2147.

\printbibliography
\end{document}